\documentclass{article}

\usepackage{microtype}
\usepackage{graphicx}
\usepackage{booktabs} 
\usepackage{multirow} 
\usepackage{hyperref}
\usepackage{graphicx}    

\usepackage[accepted]{icml2025}

\usepackage{amsmath}
\usepackage{amssymb}
\usepackage{mathtools}
\usepackage{amsthm}
\usepackage{subfig} 
\usepackage[capitalize,noabbrev]{cleveref}

\theoremstyle{plain}
\newtheorem{theorem}{Theorem}[section]
\newtheorem{proposition}[theorem]{Proposition}

\newtheorem{corollary}[theorem]{Corollary}
\theoremstyle{definition}
\newtheorem{definition}[theorem]{Definition}

\theoremstyle{remark}

\newcommand{\first}[1]{\textcolor{red}{#1}}
\newcommand{\second}[1]{\textcolor{violet}{#1}}
\newcommand{\third}[1]{\textcolor{orange}{#1}}
\usepackage{pifont}%
\newcommand{\cmark}{\ding{51}}%
\newcommand{\xmark}{\ding{55}}%
\usepackage[textsize=tiny]{todonotes}

\begin{document}

\twocolumn[
\icmltitle{GraphMinNet: Learning Dependencies in Graphs with Light Complexity Minimal Architecture}



\begin{icmlauthorlist}
\icmlauthor{Md Atik Ahamed}{a}
\icmlauthor{Andrew Cheng}{b}
\icmlauthor{Qiang Ye}{c}
\icmlauthor{Qiang Cheng}{a,d}

\end{icmlauthorlist}

\icmlaffiliation{a}{Department of Computer Science, University of Kentucky, Lexington, KY 40506, USA}
\icmlaffiliation{b}{Department of Statistics and Applied Probability, University of California, Santa Barbara, CA 93106}
\icmlaffiliation{c}{Department of Mathematics, University of Kentucky, Lexington, KY 40506, USA}
\icmlaffiliation{d}{Institute for Biomedical Informatics, University of Kentucky, Lexington, KY 40506, USA}

\icmlcorrespondingauthor{Qiang Cheng}{qiang.cheng@uky.edu}

\icmlkeywords{Machine Learning, ICML, Graph, Convolution, RNN, minGRU}

\vskip 0.3in
]



\printAffiliationsAndNotice{}  

\begin{abstract}
Graph Neural Networks (GNNs) have demonstrated remarkable success in various applications, yet they often struggle to capture long-range dependencies (LRD) effectively. This paper introduces GraphMinNet, a novel GNN architecture that generalizes the idea of minimal Gated Recurrent Units to graph-structured data. Our approach achieves efficient LRD modeling with linear computational complexity while maintaining permutation equivariance and stability. The model incorporates both structural and positional information through a unique combination of feature and positional encodings, leading to provably stronger expressiveness than the 1-WL test. Theoretical analysis establishes that GraphMinNet maintains non-decaying gradients over long distances, ensuring effective long-range information propagation. Extensive experiments on ten diverse datasets, including molecular graphs, image graphs, and synthetic networks, demonstrate that GraphMinNet achieves state-of-the-art performance while being computationally efficient. Our results show superior performance on 6 out of 10 datasets and competitive results on the others, validating the effectiveness of our approach in capturing both local and global graph structures.
\end{abstract}

\section{Introduction}
Graphs are widely used in various fields ranging from social networks and knowledge representations
to engineering. Graph neural networks (GNNs) provide crucial techniques for extracting information and
making inference over graph data. While numerous GNN models have been developed, important
challenges still need to be overcome.

A fundamental challenge in GNNs is capturing long-range dependencies (LRDs) - the ability to model relationships between nodes that are far apart in the graph structure. Classical GNNs typically use message passing between neighboring nodes, where each layer only allows information to travel one hop. To reach distant nodes, information must pass through many layers, leading to a phenomenon called over-squashing: as information propagates through multiple layers, messages from many source nodes are compressed into fixed-size node representations, causing excessive compression and loss of important details. This over-squashing problem, particularly severe in deep GNNs or highly connected graphs, creates information bottlenecks that prevent effective modeling of long-range relationships.

The ability to capture LRD is critical because many real-world graphs inherently contain important long-distance relationships and require understanding of global graph structure. Beyond over-squashing, attempting to capture these dependencies through deep message passing leads to additional challenges such as over-smoothing (where node features become too similar), gradient vanishing, and information dilution. Thus, while GNN performance often depends on capturing both local and distant graph interactions, existing approaches struggle with this fundamental tension.

Several approaches have been proposed to address these challenges, including attention mechanisms over longer paths \cite{ying2021transformers, kreuzer2021rethinking, rampavsek2022recipe}, global graph features \cite{zhang2021nested, you2021identity}, skip connections \cite{wu2022stabilizing}, graph diffusion \cite{chamberlain2021grand}, and multi-scale architectures \cite{ying2018hierarchical}. While these approaches show promise, attention-based and multi-scale methods often face computational scalability issues with large graphs, whereas simpler approaches like skip connections and global features can be prone to overfitting on complex graph structures.

To address these limitations, we propose a novel approach that achieves effective LRD modeling with linear computational complexity. Our key insight comes from recursive neural networks (RNNs), particularly an emerging variant called minGRU \cite{feng2024were} that has demonstrated remarkable ability to capture long-range dependencies in sequential data with linear complexity.

However, introducing the idea of minGRU to graph data presents fundamental challenges due to the inherent differences between sequential and graph structures. Unlike sequential data where elements have natural ordering and positional information, graph nodes lack intrinsic ordering. Moreover, graphs contain explicit structural information through edges that is absent in sequential data. To address these differences, we develop GraphMinNet, which bridges these differences by generalizing minGRU's efficient mechanisms to the graph domain while preserving its feature learning advantages and incorporating graph-specific structural information.

In summary, this paper has the following contributions: 
\begin{itemize}
\item We generalize the idea of minimal GRU, an RNN variant, to graph-structured data by developing an efficient algorithm that integrates node features with positional encoding;
\item The resulting model, GraphMinNet, has key advantages including strong ability to capture LRD, expressivity between 1-WL and 3-WL in terms of graph discrimination power, and linear computational complexity, all with provable performance guarantees;
\item Through extensive experiments across diverse datasets, we demonstrate that our algorithm has superior predictive accuracy and efficiency compared to state-of-the-art baselines. 
\end{itemize}

\section{Related Works}
\begin{figure*}[h]
    \centering
    \includegraphics[width=0.9\linewidth]{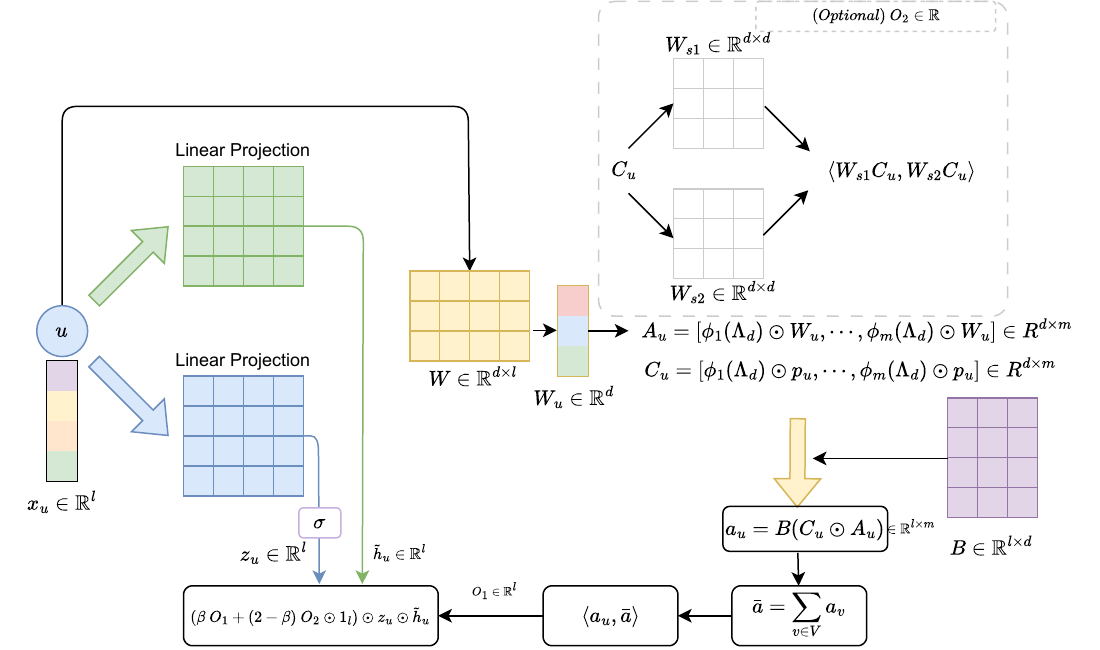}
    \caption{Schematic diagram of our proposed method. Here $u\in\mathbb{R}^{l}$ represents each node. $\sigma$ denotes the sigmoid activation function. $\beta$ represents a learnable parameter to focus on the specific part.}
    \label{fig:method}
\end{figure*}
We categorize the existing GNNs for prediction or classification tasks into several groups by their main network architectures. Notably, more recent models often combine elements from multiple categories.

{\bf{GCN or message passing based}}. These methods leverage either spatial or spectral domain operations through message passing or graph convolutions. Key approaches include Graph Convolutional Networks (GCN) \cite{kipf2016semi}, Gate GCN \cite{bresson2017residual}, Graph Isomorphism Networks (GIN) \cite{xu2018powerful}, Graph Attention Networks (GAT) \cite{velivckovic2018graph}, GraphSAGE \cite{hamilton2017inductive}, and Principal Neighborhood Aggregation (PNA) \cite{corso2020principal}. While efficient and scalable, these models typically have limited ability to capture LRD. Moreover, standard GCNs have limited expressiveness, as they are equivalent to the 1-Weisfeiler-Lehman (WL) test, whereas higher-order k-GNNs are proven to be more expressive \cite{morris2019weisfeiler}.

{\bf{Kernel based}}. Graph kernel methods have been extensively studied, including neural tangent kernel \cite{jacot2018neural}, graph neural tangent kernel \cite{du2019graph},
graph kernels with neural networks \cite{morris2019weisfeiler}, and spectral kernel learning \cite{zhi2023gaussian}. These methods offer theoretical guarantees through strong mathematical foundations from kernel theory, particularly in terms of provable convergence properties.
However, they face several challenges in adaptability to hierarchical structures, capturing complex patterns, and incorporating node/edge features. Thus, they may have limited representation power. 

{\bf{Transformer or SSM based}}. These recent models leverage Transformers or State Space Models (SSMs) to capture LRD. Following the success of Transformers in text and image domains, they have been adapted for graph learning. Key approaches include Graphormer \cite{ying2021transformers}, SAN \cite{kreuzer2021rethinking}, GraphGPS \cite{rampavsek2022recipe}, Exphormer \cite{shirzad2023exphormer}, Grit \cite{ma2023graph}, and Specformer \cite{bo2022specformer}. With the emergence of SSMs such as Mamba \cite{gu2023mamba}, new graph models like Graph-Mamba \cite{wang2024state} have been developed. While these models effectively capture LRD between distant nodes, Transformer-based models typically have quadratic complexity and are computationally demanding, whereas SSM-based models may not fully preserve the permutation equivariance of graph data \cite{zhang2024expressive}.

{\bf{Position or structural aware}}. Various techniques incorporate position or substructure information in graphs, from early approaches like DeepWalk \cite{perozzi2014deepwalk} and node2vec \cite{grover2016node2vec}, to recent work in position encoding \cite{you2019position}, distance encoding \cite{li2020distance}, structural RNN \cite{jain2016structural}, and SPE \cite{huang2024stability}. Recent models have explored substructural information through nested structures (NGNN \cite{zhang2021nested}), identity-aware patterns (ID-GNN \cite{you2021identity}), augmented kernels (GIN-AK+ \cite{zhao2021stars}), iterative learning (I2-GNN \cite{huang2022boosting}), and sign patterns (SUN \cite{frasca2022understanding}, SignNet \cite{lim2022sign}). While these techniques effectively capture higher-order interactions with provable expressiveness bounds, they face scalability challenges, e.g., due to expensive distance computations and requirements for full graph structure access. They may also suffer from generalization issues, including overfitting to specific structures and sensitivity to graph variations.

{\bf{Implicit and continuous-depth architecture}}. Continuous-depth architectures have emerged as a promising direction for graph learning, starting with GDE \cite{poli2019graph} which models network dynamics as a continuous-time system. GRAND \cite{chamberlain2021grand} developed graph neural diffusion, while Continuous GNNs \cite{xhonneux2020continuous} provided a framework for continuous-depth networks, and GRED \cite{ding2024recurrent} enhanced this with recurrent and dilated memory. These approaches offer memory efficiency and natural handling of dynamical processes, but face challenges in solving differential equations and require careful tuning to balance stability and expressiveness.

\section{Methods}

\label{sec:method}
\subsection{Preliminary}
GRU has been improved to minGRU \cite{feng2024were} to overcome gradient vanishing/explosion and enable better capture of global dependencies. For an input sequence of tokens, minGRU is defined as:
\begin{equation}
\begin{split}
h_t &= (1-z_t) \odot h_{t-1} + z_t \odot \tilde{h}_{t} \\
z_t &= \sigma(\text{Linear}{d_h} (x_t))\\
\tilde{h}_{t} &= \text{Linear}{d_h} (x_t)),
\end{split}
\end{equation}
where $x_t \in \mathbb{R}^{d_x}$ is the input feature vector and $h_t \in \mathbb{R}^{d_h}$ is its corresponding state-space representation at time $t$. Here, $\sigma(\cdot)$ is an element-wise non-linear activation function with values in $(0, 1)$, and $\text{Linear}_{d_h}(x_t)$ projects $x_t$ to a $d_h$-dimensional state space via an MLP.

The model achieves computational efficiency with complexity $O(2d_xd_h)$ \cite{feng2024were}, significantly lower than the original GRU's $O(3d_h(d_x + d_h))$. Training can be parallelized using a parallel scan algorithm \cite{feng2024were}. To enable more effective feature extraction, minGRU employs state expansion where the state dimension $d_h = \alpha d_x$ with $\alpha \ge 1$.

\subsection{GraphMinNet for Graph Learning} 

The minGRU model enhances the original GRU with several key advantages: 1) ability to capture global dependency, 2) linear efficiency in terms of input sequence length, 3) scalable model size with respect to input length, and 4) shift equivariance. 
Moreover, compared to state-space models \cite{gu2020hippo} \cite{gu2021combining}, in particular, Mamba \cite{gu2023mamba} that has a linear complexity and scalability, minGRU does not have a fixed state-space length and thus have a stronger ability to possess context awareness or content selectivity \cite{feng2024were}. These advantages offer potential solutions to common GNN challenges, motivating our development of GraphMinNet.

To develop this model, first we obtain the explicit expression of the minGRU model containing no state variable $h_t$: 
\begin{align}
   \nonumber h_t = &z_t  \odot \tilde{h}_{t} + (1-z_t) \odot h_{t-1}  \\
    \nonumber = &z_t  \odot \tilde{h}_{t}  
    + (1-z_t) \odot   z_{t-1}  \odot \tilde{h}_{t-1} \\
   \nonumber & + (1-z_t) (1-z_{t-1})  \odot  {h}_{t-2} = \cdots \\
    =  &\sum_{i=1}^t \prod_{j=i+1}^{t} (1-z_j) \odot z_i \odot \tilde{h}_i.
    \label{eq_minGRU_expand}
\end{align}
In the last equality of Eq.\ (\ref{eq_minGRU_expand}) and hereafter,  the notation $\prod$ stands for element-wise multiplication of $d_h$-dimensional vectors across relevant indexes. By denoting $c_i = \prod_{j=1}^i (1-z_j)$, $i=1, \cdots, t$,  Eq. (\ref{eq_minGRU_expand}) can be written as 
\begin{equation}
h_t = (\sum_{i=1}^{t} z_i \odot \tilde{h}_i \odot c_i^{-1}) \odot c_t.
\label{eq_expanded}
\end{equation}
The above equation can facilitate us to derive a corresponding model for graph learning, as it does not contain any intermediate, latent state variables. 

For a graph $G = (V, A)$ with $n$ nodes in set $V$ and adjacency matrix $A$, each node $u$ has a feature vector $x_u \in \mathcal{R}^l$. To introduce the idea of minGRU to graph-structured data, we make three key observations:

1) Position Indexing: While minGRU uses sequence positions as indices, we associate these with graph nodes. Since graphs lack natural ordering, we replace $\sum_{i \le t}$ with $\sum_{v \in V}$ to ensure permutation equivariance.

2) Positional Information: To capture node positions in the graph structure, we employ Laplacian positional embedding (LPE) \cite{wang2022equivariant, huang2024can, huang2023stability}. Given the graph Laplacian eigendecomposition $L = \tilde{V} \Lambda \tilde{V}^T$, where $\tilde{V} \in {\mathcal{R}^{n \times n}}$ contains the eigenvectors and $\Lambda$ is the diagonal eigenvalue matrix, we define the $d$-dimensional LPE for node $u$ as $p_u = \tilde{V}[u, 1:d] \in {\mathcal{R}^d}$. This encoding captures the absolute position of node $u$ in the graph. We denote by $\Lambda_d$ the vector of top $d$ non-zero eigenvalues.

3) Content Dependence: The interaction terms $c_i^{-1} \odot c_t$ depend on $z_j$ (for $i+1 \le j \le t$), which in turn depend on input features $x_j$, creating content-dependent state variables. This mechanism parallels the selection mechanism in Mamba \cite{gu2023mamba, ahamed2024timemachine}. To fully capture both structural and feature information, we encode positions and features separately in matrix form. Therefore, we need to replace the $\odot$ operation involving the matrices with suitable operations that produce vectors.

Based on these observations, we construct the node embedding as follows. First, we define the feature component $A_u \in \mathcal{R}^{d \times m}$ for node $u$:
\begin{align}
A_u = [\phi_1(\Lambda_d) \odot (W x_u), \cdots, \phi_{m}(\Lambda_d) \odot (W x_u)],
\label{eq_Au}
\end{align}
where $W \in \mathcal{R}^{d \times l}$ is a learnable weight matrix and $\phi_i(\cdot): \mathcal{R}^d \rightarrow \mathcal{R}^d$ are learnable permutation equivariant functions, $i=1, \cdots, m$. Here, $\phi_i(\Lambda_d)$ are permutation equivariant to the top-$d$ eigenvalues of the Laplacian, similar to the global permutation-equivariant set aggregation functions in \cite{wang2022equivariant, huang2024can}.

Similarly, we construct the positional component $C_u \in \mathcal{R}^{d \times m}$:
\begin{align}
C_u = [\phi_1(\Lambda_d) \odot p_u, \cdots, \phi_{m}(\Lambda_d) \odot p_u].
\label{eq_Cu}
\end{align}
The overall embedding for node $u$ combines these components:
\begin{align}
    a_u = B( A_u \oplus_1 C_u),
    \label{eq:au}
\end{align}
where $B \in \mathcal{R}^{l \times d}$ is a learnable matrix and $\oplus_1$ denotes element-wise aggregation (e.g., addition or multiplication) between matrices $A_u$ and $C_u$. The resulting embedding $a_u$ has size $l \times m$.

The inverse operation in minGRU is adapted to graphs by associating node encodings with quantities in Eq. (\ref{eq_minGRU_expand}) and Eq. (\ref{eq_expanded}):
$a_v \leftarrow c_j^{-1} \odot z_j \odot \tilde{h}_j$ for any node $v \in V$ and $a_v \odot z_v \odot \tilde{h}_v \leftarrow c_j$, where $\odot$ between $a_v$ and $z_v$ denotes element-wise multiplication of each column of $a_v$ with $z_v$.

Inspired by Eq. (\ref{eq_expanded}), we formulate GraphMinNet for node $u$ as: 
\begin{align}
\label{eq_graph_expanded}
    h_u &= \sum_{v \in V}  a_v \odot (a_u  \odot z_u \odot \tilde{h}_u) \\
    z_u &= \sigma(Linear (x_u))\\
    \label{eq_tilde_h}
    \tilde{h}_u &=  Linear (x_u).
\end{align}
Here, the state variable $h_u$ has the same dimension $l$ as node feature $x_u$. Our formulation differs from the graph convolution in \cite{huang2024can} in two key aspects: First, nonlinear feature dependence where $h_u$ depends nonlinearly on $x_u$ through $\tilde{h}_u$ (linear transformation), $z_u$ (gated feature attention), and $A_u$ (in $a_u$, containing $x_u$ in each column), with the gated feature attention providing automatic focus on important features. Second, while \cite{huang2024can} primarily emphasizes positional encoding, our formulation incorporates features through $A_u$, $z_u$, and $\tilde{h}_u$.

For matrix $a_v$, the $\odot$ operation with vector $z_v \odot \tilde{h}_v$ multiplies the vector element-wise with each matrix column. Defining $\bar{a} = \sum_{v \in V} a_v$, we can reformulate Eq. (\ref{eq_graph_expanded}) as:
\begin{align}
    h_u =   \langle a_u, \bar{a} \rangle \odot z_u \odot \tilde{h}_u. 
    \label{eq_reformulation_inner-product}
\end{align}
Here, we generalize the $\odot$ operation between $a_u$ and $\bar{a}$ by an inner product operation $\langle \cdot, \cdot \rangle$ 
because $a_u$ and $\bar{a}$ are matrices. As $\bar{a} = \sum_{v \in V} a_v$, we define $\langle a_u, a_v\rangle$ using four types of inner products:

{\emph{Type 1}}.  Elementary inner product between corresponding matrix rows;

{\emph{Type 2}}.  For $m=l$, elementary inner product between corresponding columns, followed by transposition into a column-vector; 

{\emph{Type 3}}.  Inner product of vectorized matrices, with result multiplied by $l$-dimensional all-ones vector;

{\emph{Type 4}}.  Either $\langle BC_u, BC_v\rangle \langle A_u, A_v\rangle$ or $\langle BA_u, BA_v\rangle$ $ \langle C_u, C_v\rangle$, where the first inner product uses any of Types 1-3 and the second uses Type 2 or 3.

Note that the Einstein summation operation (einsum) in \texttt{torch} can be used to efficiently calculate these different types of products. In the model, all nodes $v$ share the same weight matrices $B$ (in $a_v$) and $W$ (in $A_v$). Because emphasizing on node $u$ potentially facilitates capturing its useful information, 
we may further consider the following formulation:
\begin{align}
\nonumber     h_u =  & z_u \odot \tilde{h}_u  \odot (\beta \langle a_u, \bar{a}\rangle \\
&+ (2-\beta) \langle W_{s1} C_u, W_{s2} C_u \rangle 1_{l}),
\label{eq_GraphMinNet_with_self}
\end{align}
where $\beta \in [0, 2]$ is a weighting parameter, $W_{s1}, W_{s2} \in \mathcal{R}^{d \times d}$ are learnable matrices, and $1_{l} \in \mathcal{R}^l$ is an all-ones vector.
The inner product $\langle \cdot, \cdot \rangle$ in the second term uses Type 3 (matrix vectorization).
This formulation generalizes Eq. (\ref{eq_graph_expanded}), which can be recovered as a special case when $\beta = 1$, showing how the self-loop term integrates with the original structure.

\subsection{Properties of GraphMinNet}
We present several key properties of our GraphMinNet formulation. Detailed proofs are provided in the appendix. 
\begin{definition} [Permutation Equivariance]
For a graph $G$ with node features $X$ and adjacency matrix $A$, given a permutation matrix $Q$, a GNN function $f$ is permutation equivariant if and only if $f(QX, QAQ^T) = Qf(X, A)$.
\end{definition}
This property ensures that when nodes of an input graph are reordered (permuted), the node-level outputs transform according to the same permutation. Since node ordering in graphs is arbitrary, permutation equivariance provides an essential inductive bias that enables the model to learn representations independent of node indexing.

\begin{definition} [(Lipschitz) Stability of a GNN]
For a GNN $f$, input graphs $G$ and $G'$, and distance metrics $d_i$ and $d_o$ for graphs and outputs respectively, if $d_i(G, G') \le \epsilon$, then $d_o(f(G), f(G')) \le L\epsilon$, where $L$ is the Lipschitz constant of the network.
\end{definition}
Stability encompasses three aspects: Structural stability refers to how outputs change when edges are added/removed. Thus, it is about what the output response is to changes in graph connectivity.   Feature stability refers to how outputs change when node/edge features are perturbed, thus is about sensitivity to noise in feature values.  Spectral stability refers to how changes in the graph's eigenvalues affect the output, which is particularly important for spectral-based GNN approaches. 
As shown in \cite{huang2023stability}, stability is more general than equivariance \cite{huang2023stability}  and implies generalization ability \cite{bousquet2002stability} \cite{shalev2010learnability}. Therefore, the stability property is critical to the GNN algorithm's robustness to noise in real-world data, generalization abilities, and adversarial robustness. 
 
\begin{proposition}
The formulation of GraphMinNet in Equations (\ref{eq_graph_expanded}) (or (\ref{eq_GraphMinNet_with_self})) 
through (\ref{eq_tilde_h}) is permutation equivariant. Moreover, if the functions $\phi_i(\cdot)$ used in forming $A_u$ and $C_u$ (Equations (\ref{eq_Au}) and (\ref{eq_Cu})) are Lipschitz, then GraphMinNet is Lipschitz stable with respect to both feature vectors and eigenvalues.
\label{pro_equivariance}
\end{proposition}

Many common permutation equivariant functions in neural networks are naturally Lipschitz, including linear permutation equivariant functions, element-wise Lipschitz operations, max/min pooling, and mean pooling. Thus, the Lipschitz condition is readily satisfied for typical choices of $\phi(\cdot)$, ensuring provable stability of GraphMinNet. 
As stability implies strong generalizability and robustness, we have the following result. 
\begin{corollary}
When the condition in Proposition \ref{pro_equivariance} is satisfied, GraphMinNet has provable generalization ability and robustness.
\end{corollary}

The following property establishes GraphMinNet's ability to capture long-range dependencies, which is critical for effective graph learning.
\begin{proposition}[Long-range Dependency]
There exists $\phi(\cdot)$ such that the gradient norm $\| \frac{\partial h_u}{\partial x_v} \|$ of GraphMinNet does not
decay as $spd(u, v)$ grows (with $n$ tending to $\infty$), where $spd(u, v)$ is the shortest path distance between $u$ and $v$.
\label{pro_lrd}
\end{proposition}

In this paper, we assume the eigenvalue decomposition of $A$ is pre-computed. For large sparse graphs, we adopt the Lanczos algorithm (e.g., implemented in ARPACK) to efficiently compute the $d$ largest/smallest eigenvalues and eigenvectors of the adjacency/Laplacian matrices. This computation has complexity $O(Md)$, where $M$ is the number of edges. Given these pre-computed eigenvalues and eigenvectors, the complexity and scalability of our algorithm is given as follows:
\begin{proposition}[Complexity and Scalability]
The hidden states $h_1, \cdots, h_n$ can be computed from $x_1, \cdots, x_n$ with a complexity of $O(nmdl)$, where $n$ is the number of nodes, $m$ is the number of columns in node encoding ($A_u$, $C_u$, or $c_u$),
$d$ is the dimension of rows in node encoding, and $l$ is the feature dimension.

Additionally, the GraphMinNet algorithm in Equations (\ref{eq_graph_expanded}) or (\ref{eq_GraphMinNet_with_self}) achieves linear scalability with respect to the number of nodes.
\label{pro_complexity}
\end{proposition}

\begin{proposition}[Expressiveness]
The formulation of GraphMinNet is more powerful than the 1-WL test but not more powerful than the 3-WL test.
\label{pro_express}
\end{proposition}
These properties collectively demonstrate that GraphMinNet achieves efficient long-range dependency modeling with linear complexity while maintaining strong expressive power between 1-WL and 3-WL tests. 

\section{Experiments}
\begin{table*}[t!]
\centering
\caption{\label{tab:main_result}Performance comparison on various datasets. Best results are represented in colors by $\first{\mathbf{first}}$, $\second{\mathbf{second}}$, and $\third{\mathbf{third}}$. Results are reported as $\text{mean}_{\pm \text{std}}$.}
\resizebox{\linewidth}{!}{%
\begin{tabular}{lcccccccccccccc}
\toprule
\multicolumn{1}{c}{\multirow{2}{*}{}} & MNIST & CIFAR10 & PATTERN & CLUSTER  & Molhiv & PascalVOC-SP & Peptides-func & Peptides-struct & ZINC & MalNet-Tiny & Avg. Rank\\
\cmidrule(l{5pt}r{5pt}){2-12}
\multicolumn{1}{c}{}                      & Accuracy $\uparrow$ & Accuracy $\uparrow$ & Accuracy $\uparrow$ & Accuracy $\uparrow$  & AUROC $\uparrow$ & F1 score$\uparrow$  & AP $\uparrow$  & MAE $\downarrow$ & MAE $\downarrow$ & Accuracy $\uparrow$&Lower better$\downarrow$\\
\midrule
GCN & $90.705_{\pm 0.218}$ & $55.710_{\pm 0.381}$ & $71.892_{\pm 0.334}$ &  $68.498_{\pm 0.976}$ & $75.99_{\pm1.19}$  & $0.1268_{\pm 0.0060}$ & $0.5930_{\pm 0.0023}$ & $0.3496_{\pm 0.0013}$ &$0.367_{\pm 0.011}$ & ${81.0}$&$9.70$\\
GIN & $96.485_{\pm 0.252}$ & $55.255_{\pm 1.527}$ & $85.387_{\pm 0.136}$ & $64.716_{\pm 1.553}$ & $77.07_{\pm 1.49}$ & $0.1265_{\pm 0.0076}$ & $0.5498_{\pm 0.0079}$ &$ 0.3547_{\pm 0.0045}$ & $0.526_{\pm 0.051}$&$88.98_{\pm0.56}$&$9.90$\\
GAT & $95.535_{\pm 0.205}$ & $64.223_{\pm 0.455}$ & $78.271_{\pm 0.186}$ & $70.587_{\pm 0.447}$ & $-$ & $-$ &$-$&$-$ & $0.384_{\pm 0.007}$& $92.10_{\pm0.24}$&$9.50$\\
GatedGCN & $97.340_{\pm 0.143}$ & $67.312_{\pm 0.311}$ & $85.568_{\pm 0.088}$ & $73.840_{\pm 0.326}$ & $-$ & $0.2873_{\pm 0.0219}$ & $0.5864_{\pm 0.0077}$ &$ 0.3420_{\pm 0.0013}$  &$-$&$92.23_{\pm0.56}$&$8.25$\\
\midrule
SAN & $-$ & $-$ & $86.581_{\pm 0.037}$ & $76.691_{\pm 0.650}$ & $77.85_{\pm2.47}$ & $0.3216_{\pm 0.0027}$ & $0.6439_{\pm 0.0075}$ &$0.2545_{\pm 0.0012}$ &$0.139_{\pm 0.006}$&$-$ &$7.43$\\
GraphGPS & $98.051_{\pm 0.126}$ & $72.298_{\pm 0.356}$ & $86.685_{\pm 0.059}$ & $78.016_{\pm 0.180}$  & $\third{\mathbf{78.80}_{\pm 1.01}}$  & $0.3748_{\pm 0.0109}$ & $0.6535_{\pm 0.0041}$ &$ 0.2500_{\pm 0.0005}$&$0.070_{\pm 0.004}$ &$93.50_{\pm0.41}$&$5.70$\\
Exphormer &  $\second{\mathbf{98.550}_{\pm 0.039}}$ & ${74.690}_{\pm 0.125}$ & ${86.740}_{\pm 0.015}$ & ${78.070}_{\pm 0.037}$ & ${78.79}_{\pm1.31}$ & ${0.3975}_{\pm 0.0037}$ & $0.6527_{\pm 0.0043}$ & $0.2481_{\pm 0.0007}$  &$0.111_{\pm 0.007}$&$\second{\mathbf{94.02}_{\pm0.21}}$&$4.70$\\
Grit & $98.108_{\pm 0.111}$ &${76.468}_{\pm 0.881}$ &$ \third{\mathbf{87.196}_{\pm 0.076}}$ &$ \first{\mathbf{80.026}_{\pm 0.277}}$ & $-$  & $-$ & ${0.6988}_{\pm 0.0082}$&
${0.2460}_{\pm 0.0012}$ &$\first{0.059_{\pm 0.002}}$&$-$&$\third{3.29}$ \\
GRED & $98.383_{\pm0.012}$ & $\third{\mathbf{76.853}_{\pm0.185}}$ & ${86.759}_{\pm0.020}$ & ${78.495}_{\pm0.103}$  & $-$ & $-$ & $\second{\mathbf{0.7133}_{\pm0.0011}}$ & $\second{\mathbf{0.2455}_{\pm0.0013}}$&$0.077_{\pm 0.002}$&$-$ &$3.71$\\
Graph-Mamba & $98.420_{\pm 0.080}$ & $73.700_{\pm 0.340}$ & $86.710_{\pm 0.050}$ & $76.800_{\pm 0.360}$ & $78.23_{\pm1.21}$ & $\third{\mathbf{0.4191}_{\pm 0.0126}}$ & ${0.6739}_{\pm 0.0087}$ &$ {0.2478}_{\pm 0.0016}$  &$0.067_{\pm 0.002}$&$93.40_{\pm0.27}$&$5.00$\\
GSSC & $\third{{98.492}_{\pm 0.051}}$ & $\second{\mathbf{77.642}_{\pm 0.456}}$  &
$\second{\mathbf{87.510}_{\pm 0.082}}$& 
$\third{\mathbf{79.156}_{\pm 0.152}}$&
$\second{\mathbf{80.35}_{\pm 1.42}}$ &  
$\first{\mathbf{0.4561}_{\pm 0.0039}}$ & 
$\third{\mathbf{0.7081}_{\pm 0.0062}}$ & $\third{\mathbf{0.2459}_{\pm 0.0020}}$&$\third{0.064_{\pm 0.002}}$&$\first{\mathbf{94.06}_{\pm0.64}}$ & \second{$2.30$}\\
\midrule
Ours & $\first{\mathbf{98.598}_{\pm 0.138}}$&$\first{\mathbf{78.068}_{\pm0.785}}$&$\first{\mathbf{87.552}_{\pm0.123}}$&$\second{\mathbf{79.284}_{\pm 0.122}}$&$\first{\mathbf{80.86}_{\pm0.56}}$&$\second{\mathbf{0.4352}_{\pm0.0030}}$&$\first{\mathbf{0.7182}_{\pm0.0024}}$&$\first{\mathbf{0.2438}_{\pm0.0014}}$& $\second{0.063_{\pm 0.001}}$&$\third{\mathbf{93.72}_{\pm0.29}}$ &\first{$1.50$}\\
\bottomrule
\end{tabular}
}
\end{table*}
In this section, we present our experimental results. Table~\ref{tab:main_result} summarizes the performance of our proposed approach compared against several robust baselines across multiple datasets. We evaluated our method on 10 diverse datasets, with detailed descriptions provided in Appendix~\ref{appendix:dataset} (Table~\ref{tab:dataset}).

Our evaluation encompasses several key dataset categories: 1) The Long Range Graph Benchmark~\citep{dwivedi2022long}, which requires effective reasoning over long-range interactions, consisting of three tasks: Peptides-func (graph-level classification with 10 peptide functional labels), Peptides-struct (graph-level regression of 11 molecular structural properties), and PascalVOC-SP (superpixel classification in image graphs). 2) Molecular graph datasets, including ZINC~\citep{dwivedi2023benchmarking} for graph regression of molecular properties and ogbg-molhiv~\citep{hu2020open} with 41k molecular graphs for classification. 3) Image graph datasets MNIST and CIFAR10~\citep{dwivedi2023benchmarking}, represented as 8-nearest neighbor superpixel graphs for classification. 4) Synthetic graph datasets PATTERN and CLUSTER~\citep{dwivedi2023benchmarking}, generated using the Stochastic Block Model for node-level community classification. 5) Function call graphs from MalNet-Tiny~\citep{freitas2020large} for classification tasks. Our experimental results across these diverse settings demonstrate our method's robustness and versatility.

We report the experimental results in Table~\ref{tab:main_result}, following the evaluation protocols from GSSC~\cite{huang2024can} with mean and standard deviation across five random seeds (0 to 4). We report relevant hyperparameters in Appendix~\ref{sec:hyerparameters}, Table~\ref{tab:hyp}.
The results demonstrate GraphMinNet's superior performance, achieving the best results on 6 out of 10 datasets and ranking second on 3 datasets, leading to the highest overall average rank among all methods. Even on the remaining dataset, GraphMinNet shows competitive performance comparable to state-of-the-art baselines. 

We further evaluate GraphMinNet's efficiency in terms of both model parameters and computational cost. Table~\ref{tab:params} compares parameter counts against recent baselines, while Fig.~\ref{fig:run_time} shows runtime analysis. Our model shows better or comparable efficiency to SOTA models.  
\begin{table*}[t!]
\centering
\caption{\label{tab:params}Model parameter count comparison across datasets. The lowest and second-lowest parameter counts are highlighted in $\first{\mathbf{first}}$ and $\second{\mathbf{second}}$, respectively.}
\resizebox{\linewidth}{!}{%
\begin{tabular}{lccccccccccccc}
\toprule
\multicolumn{1}{c}{\multirow{2}{*}{}} & MNIST & CIFAR10 & PATTERN & CLUSTER  & Molhiv & PascalVOC-SP & Peptides-func & Peptides-struct & ZINC &MalNet-Tiny\\
\cmidrule(l{5pt}r{5pt}){2-11}
GraphGPS & \second{115.39K}&\second{112.73K}&\first{337.20K}&502.05K&558.63K&510.45K&504.36K&504.46K&\second{423.72K}&\third{527.24K}\\
Grit &\first{102.14K}&\first{99.486K}& 477.95K&\second{432.21K}&-&-&443.34K&438.83K&473.47K&-\\
GSSC &133k&131k&539k&539k&\second{351k}&\first{375k}&\second{410k}&\second{410k}&436k&\second{299K}\\
\midrule
Ours &122.82K&120.13K&\second{431.54K}&\first{432.01K}&\first{338.80K}&\second{474.10K}&\first{386.32K}&\first{391.92K}&\first{415.28K}&\first{279.17K}\\
\bottomrule
\end{tabular}
}
\end{table*}
\begin{figure}[ht]
    \centering
    \includegraphics[width=0.6\linewidth]{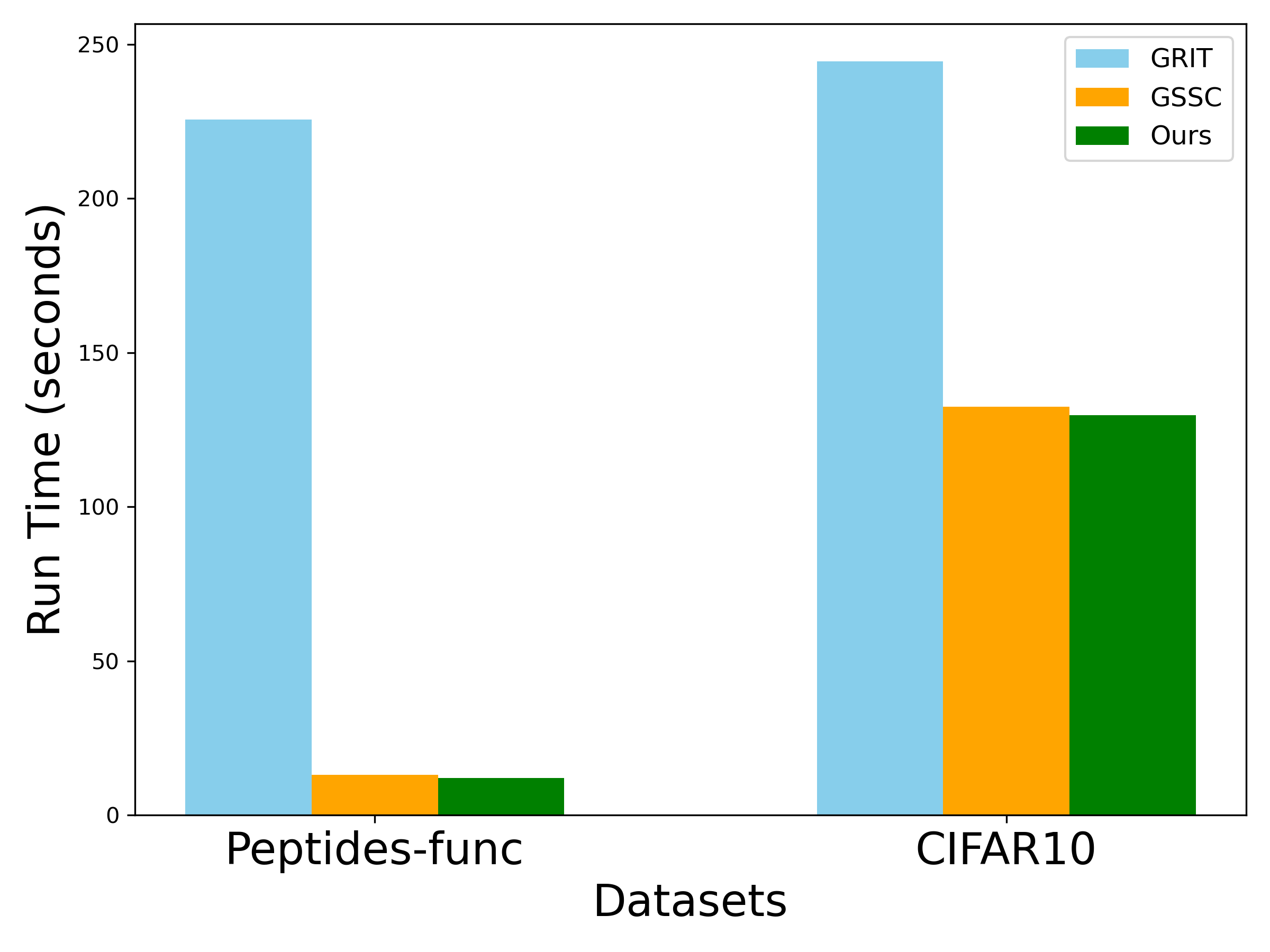}
    \caption{Run time comparison per epochs including train, validation, and test phases.}
    \label{fig:run_time}
\end{figure}

To demonstrate scalability, we conducted experiments on randomly generated graphs with node counts ranging from 1,000 to 20,000. These graphs were generated using \texttt{PyTorch Geometric}~\citep{pytorch_geometric} with an average node degree of 5, maintaining realistic sparsity constraints. Node features were randomly initialized, and we included eigenvalues, eigenvectors, and logarithmic degree values to simulate diverse graph properties.
We evaluated scalability through two metrics: FLOPs (Floating Point Operations), computed using \texttt{thop}~\footnote{https://pypi.org/project/thop}, and Maximum Memory Utilization, measured via  \texttt{torch.cuda.max\_memory\_allocated}. As shown in Figure~\ref{fig:scalability}, both FLOPs and memory utilization demonstrate linear growth with respect to the number of nodes. This linear scalability confirms our theoretical analysis and demonstrates our method's practical efficiency for large-scale graph applications.

\begin{figure}[ht]
    \centering
    \subfloat[FLOPs Utilization]{{\includegraphics[trim={0.25cm 0.25cm 0.3cm 0.25cm},clip,width=0.50\linewidth]{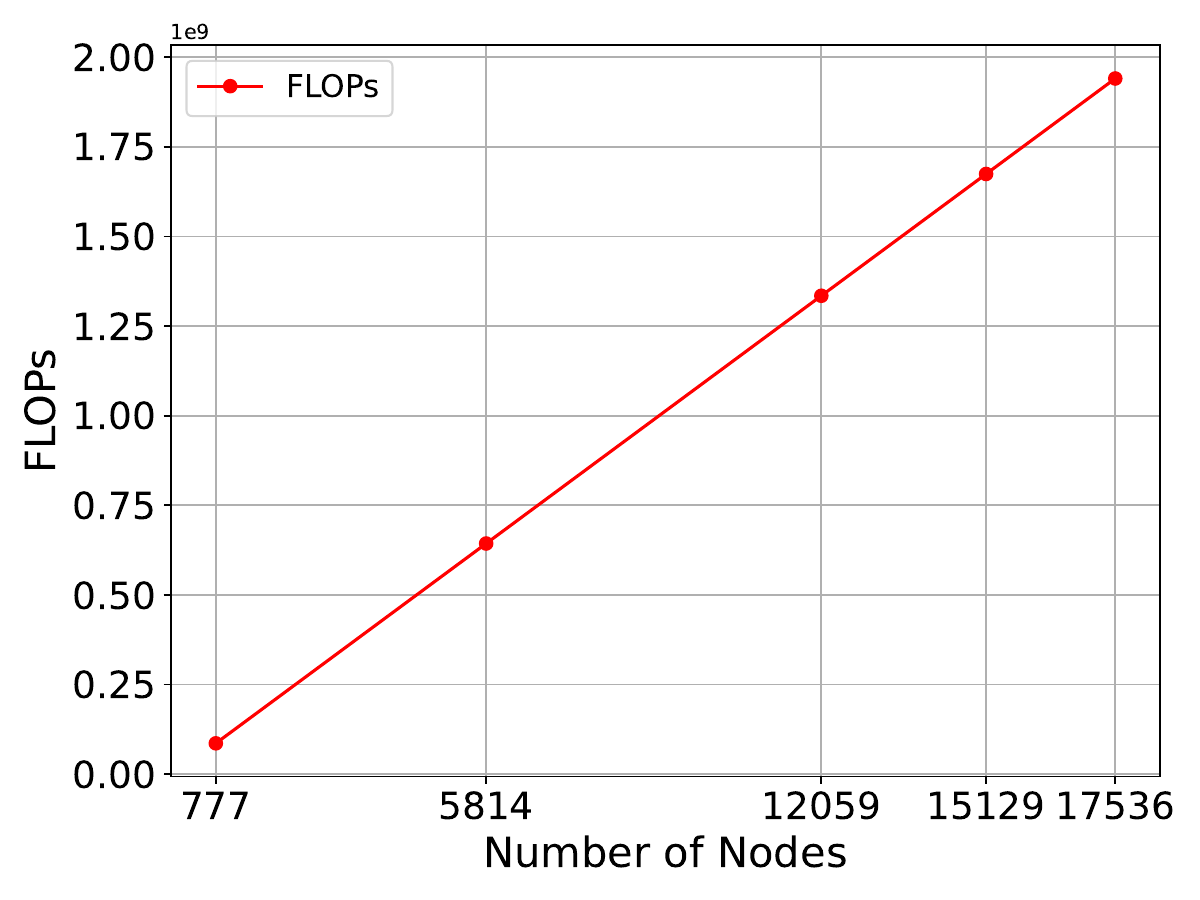} }}%
    \subfloat[Memory Utilization]{{\includegraphics[trim={0.25cm 0.25cm 0.3cm 0.85cm},clip,width=0.50\linewidth]{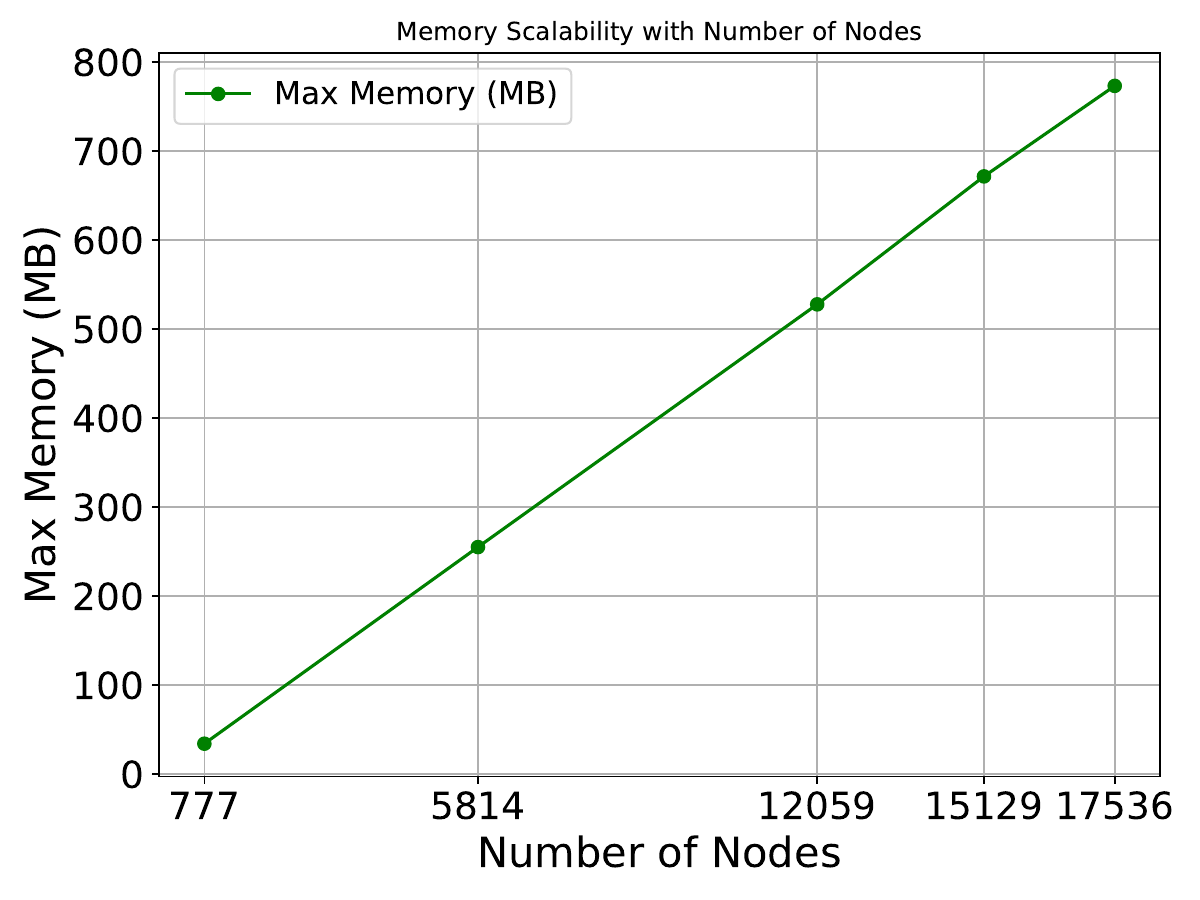} }}%
    \caption{Scalability analysis of GraphMinNet. (a) shows the linear growth of FLOPs, reflecting computational efficiency with increasing graph size, and (b) depicts the linear growth of maximum memory usage, demonstrating feasible memory requirements.}
    \label{fig:scalability}
\end{figure}

\subsection{Robustness Analysis}
To validate our theoretical stability results empirically, we evaluate our method's robustness to feature perturbations. We introduce controlled synthetic noise to node embeddings by adding Gaussian perturbations proportional to the feature magnitudes. For each node $x_u \in \mathbb{R}^l$, we compute the perturbed embedding $x'_u$ as:
\begin{equation}
x'_u = x_u + \epsilon \cdot n_u,
\end{equation}
where $\epsilon$ is the noise level, and $n_u \in {\mathcal{R}}^l$ is a noise term. 
We consider two types of noise: additive white noise $n_u = \mathcal{N}(0, I)$ and signal-dependent noise $n_u = \mathcal{N}(0, I) \odot \mu(x)$. Here, 
$\mathcal{N}(0, I)$ is standard multivariate Gaussian, and $\mu(x)$ denotes the mean feature magnitude. The second type models perturbations that scale appropriately with the underlying feature values.
Figure~\ref{fig:robustness} shows our method's performance on Molhiv and Peptides-Func datasets under increasing noise levels (0\%, 5\%, ..., 30\%). The results demonstrate consistent performance across noise levels, empirically confirming our theoretical stability guarantees.

\begin{figure}[ht]
    \centering
    \includegraphics[width=0.6\linewidth, clip]{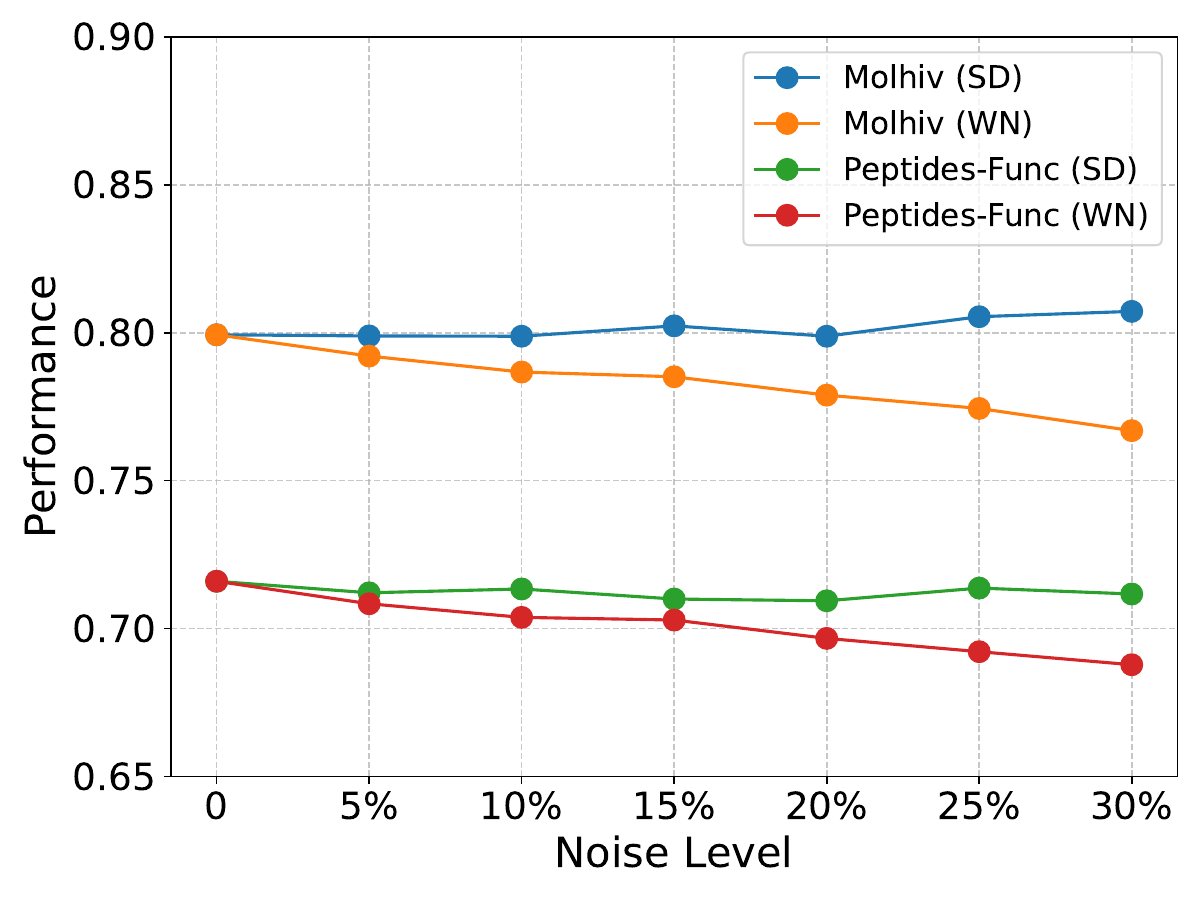}
    \caption{Robustness analysis of our model under varying noise levels. SD: signal-dependent; WN: white noise. }
    \label{fig:robustness}
\end{figure}

\section{Ablation Studies}
We conducted ablation studies to analyze key components of GraphMinNet, including the optional self-term contribution, dropout effectiveness, and local method impact.
\subsection{Optional Self-term}
We first investigate the impact of the self-term introduced in Eq. \eqref{eq_GraphMinNet_with_self}. Results in Table~\ref{tab:self_term} show dataset-dependent effects: while the self-term improves performance on Molhiv (AUROC increase of 2.55\%), it slightly decreases performance on CLUSTER and Peptides-func. This variation suggests the self-term's effectiveness correlates with specific graph properties - it appears more beneficial for molecular graphs with complex local structures (like Molhiv) compared to more regular graph structures (like CLUSTER).
\begin{table}[ht]
    \centering
    \resizebox{\linewidth}{!}{
    \begin{tabular}{lccc}
    \toprule
    Self-term & CLUSTER & Peptides-func & Molhiv\\
    \midrule
    \xmark & ${79.284}_{\pm 0.122}$ & ${0.7182}_{\pm 0.0024}$&${78.31}_{\pm 1.06}$\\
    \cmark & ${78.942}_{\pm 0.126}$&${0.7020}_{\pm0.0100}$& ${80.86}_{\pm 0.56}$\\
    \bottomrule
    \end{tabular}
  }
    \caption{Ablation on optional self-term in our method.}
    \label{tab:self_term}
\end{table}
\subsection{Embedding Representation}
We analyzed different methods for aggregating node embeddings $a_u$ in Eq.~\ref{eq:au}, comparing two element-wise operations for combining feature and positional information: multiplication and addition. Our results showed that element-wise multiplication achieves higher accuracy across most datasets (e.g., two datasets shown in Table~\ref{tab:au}). Based on these results, we adopt element-wise multiplication by default except for MalNet-Tiny, which uses element-wise addition.
\begin{table}[ht]
    \centering
    \begin{tabular}{lcc}
    \toprule
    Aggregation & Peptides-func & Molhiv\\
    \midrule
    Multiplication&${0.7182}_{\pm0.0024}$&${80.86}_{\pm0.56}$\\
    Addition&${0.6941}_{\pm0.0028}$ & ${80.19}_{\pm0.32}$\\
    \bottomrule
    \end{tabular}
    \caption{Comparing aggregation operation $\oplus_1$.}
    \label{tab:au}
\end{table}
\subsection{Effectiveness of Dropout Regularization}
We analyze the impact of dropout regularization on model performance. Figure~\ref{fig:dropout} compares training and validation performance with and without dropout. The results demonstrate that dropout plays a crucial role in preventing overfitting - models without dropout show significant performance degradation on validation sets, particularly on complex datasets like Peptides-func and Molhiv. This confirms dropout's importance as a regularization mechanism in our architecture.
\begin{figure}[t]
    \centering
    \includegraphics[width=0.55\linewidth]{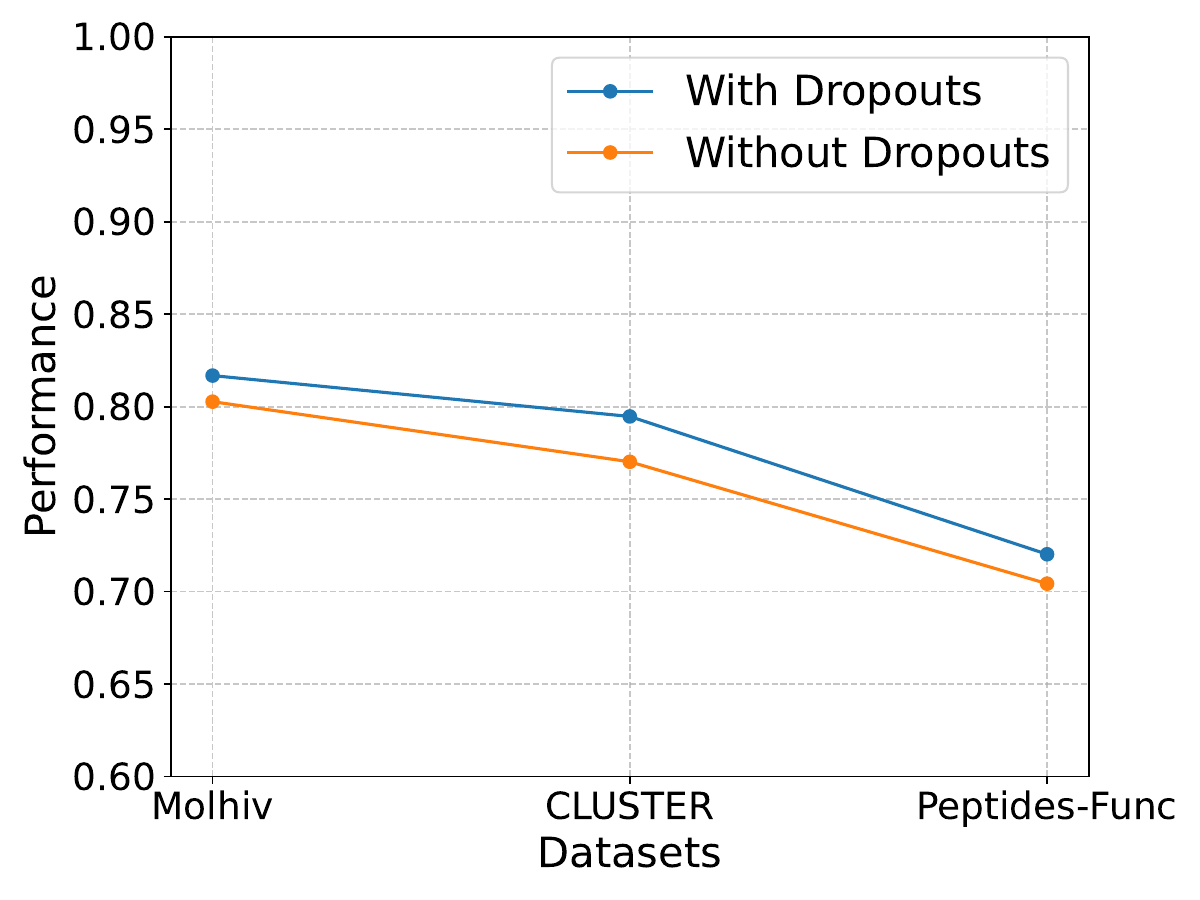}
    \caption{Effectiveness of dropouts.}
    \label{fig:dropout}
\end{figure}
\subsection{Local Method}
\begin{table}[ht]
    \centering
    \resizebox{\linewidth}{!}{
    \begin{tabular}{lccc}
    \toprule
    Method & CLUSTER & Peptides-func & Molhiv\\
    \midrule
    +GatedGCN & ${79.284}_{\pm 0.122}$&${0.7182}_{\pm0.0024}$&${80.86}_{\pm0.56}$\\
    +GINE&${77.200}_{\pm 0.415}$&${0.7129}_{\pm0.0037}$&${80.33}_{\pm0.83}$\\
    \bottomrule
    \end{tabular}
  }
    \caption{Ablation results with an additional local method.}
    \label{tab:local}
\end{table}
We evaluate GraphMinNet's stability when combined with different local methods that use local neighborhood information including edge features, such as GatedGCN and GINE. Table~\ref{tab:local} shows that GraphMinNet maintains consistent performance across different local methods: performance variations remain within about $2\%$ on CLUSTER, $0.6\%$ on Peptides-func, and within about $0.6\%$ on Molhiv. These minimal differences across diverse datasets demonstrate that GraphMinNet can effectively integrate local structural information regardless of the specific local method employed, confirming its architectural robustness and versatility.
\subsection{Nonlinear $W$ and $B$}
\label{sub:non_w_b}
\begin{table}[ht]
    \centering
    \resizebox{\linewidth}{!}{
    \begin{tabular}{lccc}
    \toprule
    Method & CLUSTER & Peptides-func & Molhiv\\
    \midrule
    Linear & ${79.284}_{\pm 0.122}$ & ${0.7182}_{\pm 0.0024}$&${80.86}_{\pm 0.56}$\\
    Nonlinear&${78.672}_{\pm 0.073}$&${0.6526}_{\pm0.0067}$&${79.76}_{\pm1.16}$\\
    \bottomrule
    \end{tabular}
  }
    \caption{Ablation on linear versus nonlinear $W$ and $B$.}
    \label{tab:non_linear}
\end{table}
In our method, we utilized single layer linear transformation for matrix $W$ and $B$. However, to explore further, instead of only using linear projection, we perform ablation study with two linear layers and a nonlinear activation such as ReLU/GELU/SiLU in between in both $W$ and $B$ matrix. We report the results in Table~\ref{tab:non_linear}. From this, we can see that, for Molhiv and CLUSTER, the performance variation is around $1\%$. However, for Peptides-func, performance drops around $7\%$. Therefore increasing more layers may lead to overfitting performance for graph data. Particularly for smaller datasets like Peptides-func.
\subsection{Projection Strategy on $z_u$ and $\tilde{h}_u$}
\begin{table}[ht]
    \centering
    \resizebox{\linewidth}{!}{
    \begin{tabular}{lccc}
    \toprule
    Method & CLUSTER & Peptides-func & Molhiv\\
    \midrule
    LP & ${79.284}_{\pm 0.122}$ & ${0.7182}_{\pm 0.0024}$&${80.86}_{\pm 0.56}$\\
    NLP &${78.980}_{\pm 0.239}$&${0.7038}_{\pm0.0040}$&${79.95}_{\pm1.34}$\\
    \bottomrule
    \end{tabular}
  }
    \caption{Ablation study comparing linear versus nonlinear projection utilized to achieve $z_u$ and $\tilde{h}_u$.}
    \label{tab:nlp}
\end{table}
Inspired by minGRU~\citep{feng2024were}, we adopt a linear projection (LP) for both $z_u$ and $\tilde{h}_u$ in our method. However, to further explore the projection strategy, we conduct an ablation study by employing a nonlinear projection (NLP), similar to Subsection~\ref{sub:non_w_b}. The results presented in Table~\ref{tab:nlp} indicate that NLP maintains competitive performance and may be beneficial for datasets with higher-dimensional node features to capture complex patterns. However, since NLP does not provide significant improvements while introducing additional parameters, we adopt LP as the default projection strategy for all datasets.

\section{Conclusion}
This paper presents GraphMinNet, a novel graph neural network that effectively captures long-range dependencies while maintaining linear computational complexity. Our approach successfully generalizes the idea of minimal GRU to graph-structured data while preserving permutation equivariance and ensuring stability, with theoretical guarantees for non-decaying gradients over long distances.
Our key contributions include an efficient integration of node features with positional encoding, achieving linear complexity and scalability, along with theoretical proofs establishing the model's stability and expressiveness bounds between 1-WL and 3-WL tests. Extensive experimental results across diverse datasets demonstrate GraphMinNet's effectiveness, achieving superior performance while maintaining computational efficiency.
Future directions include extensions to dynamic graphs, applications to larger-scale networks, and adaptation to heterogeneous graph structures.

One limitation of our model is edge features are not explicitly accounted for. In our future study, we would like to incorporate edge feature directly in our model.
\section*{Impact Statement}
This work aims to advance machine learning methods for graph-structured data. While our technical contributions focus on graph neural networks, we acknowledge that ML systems can have broader societal impacts. Potential applications of our work include modeling social networks, analyzing biological systems, and understanding complex network interactions. However, careful consideration must be given to bias in graph construction and the representativeness of training data in deployment contexts. We particularly emphasize the importance of responsible data collection and proper validation when applying these methods to sensitive domains. We encourage future work to investigate these aspects and develop robust guidelines for ethical applications.
\section*{Acknowledgements}
This work was supported in part by the NSF under Grants IIS 2327113 and ITE 2433190; and the NIH under Grants P30AG072946. We would like to thank the NSF support for AI research resource with NAIRR240219. We would like to thank the University of Kentucky Center for Computational Sciences and Information Technology Services Research Computing for their support and use of the Lipscomb Compute Cluster and associated research computing resources.
\bibliography{2.references}
\bibliographystyle{icml2025}

\section{Appendix}
\subsection{Proofs of Model Properties}
\begin{proposition}
The GraphMinNet in Eq. (\ref{eq_graph_expanded}) is permutation equivariant. Moreover, if 
$\phi_i(\cdot)$ are Lipschitz, the GraphMinNet is also stable in terms of features and eigen values. 
\end{proposition}
\begin{proof}
We first establish three fundamental facts, which can be straightforwardly proved: 

\textbf{Fact 1.} Permutation equivariance holds for node-wise operations that are applied independently to each node's features.

\textbf{Fact 2.} The composition of permutation equivariant functions is also permutation equivariant.

\textbf{Fact 3.} The composition of Lipschitz functions is also Lipschitz.

\medskip
\textit{Permutation Equivariance:}\\
By Fact 1, $z_u$ and $\tilde{h}_u$ in Eq. (\ref{eq_graph_expanded}) are permutation equivariant, as both are node-wise operations. Next, we prove 
that $A_u$, $C_u$ and $h_u$ are also permutation equivariant. For a node permutation $Q$, let $q(u)$ denote the index of node $u$ after permutation. 
Then, the Laplacian for the permuted graph is 
\[
\hat{L} = Q L Q^T = (Q {\tilde{V}}) \Lambda  (Q {\tilde{V}})^T.
\] 

After node permutation, $\Lambda$ does not change. Then, we have the 
hat-version $\hat{p}_u$, $\hat{A}_u$, $\hat{C}_u$, and $\hat{h}_u$ for 
$\hat{L}$, which are counterparts of the corresponding quantities for $L$. 
Therefore:
\begin{align*}
\hat{p}_{q(u)} &= (Q \tilde{V}) [q(u), 1:d] = \tilde{V}[u, 1:d] = p_u,\\
\hat{A}_{q(u)} &=[\phi_1(\Lambda_d)\odot (W\hat{x}_{q(u)}), \cdots,  \phi_m(\Lambda_d)\odot (W\hat{x}_{q(u)})] \\
                &= [\phi_1(\Lambda_d)\odot (W x_{u}), \cdots,  \phi_m(\Lambda_d)\odot (Wx_{u})] = A_u,\\
\hat{C}_{q(u)} &=[\phi_1(\Lambda_d)\odot (\hat{p}_{q(u)}), \cdots,  
\phi_m(\Lambda_d)\odot (\hat{p}_{q(u)})] \\
                &= [\phi_1(\Lambda_d)\odot p_{u}, \cdots,  \phi_m(\Lambda_d)\odot p_{u}] = C_u.
\end{align*}

Thus, $\hat{a}_{q(u)} = a_u$. Finally, we have:  
\[
\begin{split}
\hat{h}_{q(u)} &= \sum_{v \in V} \hat{a}_{q(v)} \odot( \hat{a}_{q(u)} \odot \hat{z}_{q(u)} \odot \hat{\tilde{h}}_{q(u)}) \\
&= \sum_{v \in V} {a}_{v} \odot( {a}_{u} \odot {z}_{u} \odot {\tilde{h}}_{u})\\
&= h_u.
\end{split}
\]
Therefore, the GraphMinNet algorithm is permutation equivariant.

\medskip

\textit{Stability Analysis:}\\
First, we consider the feature stability. For any node $u$ and its feature vector $x_u$ with perturbation $\Delta x_u$,
\begin{itemize}
    \item $z_u = \sigma(c_1 x_u)$ is Lipschitz since $\sigma$ is Lipschitz
    \item $\tilde{h}_u = c_2 x_u$ is clearly Lipschitz with constant $|c_2|$.
\end{itemize}
By Fact 3 and chain rule, $h_u$ is Lipschitz as composition of Lipschitz functions.

For spectral and structural stability, consider a symmetric perturbation matrix $E$ to adjacency matrix $A$. By Weyl's inequality, we have
\[
|\lambda_i (A + E) - \lambda_i(A)| \le \|E\|_2.
\]
Since $\phi_i(\cdot)$ are Lipschitz by assumption with constants $L_i$, and $A_u$ involves composition of $\phi_i$ with eigenvalues, we have
\[
\|\Delta A_u\| \leq \max_i L_i \|E\|_2.
\]
Finally, as $h_u$ is Lipschitz in $a_u$ and $a_u$ is Lipschitz in $A_u$, $h_u$ is stable with respect to both feature and eigenvalue perturbations.
\end{proof}

\begin{proposition}
There exists $\phi(\cdot)$ such that the gradient norm $\|\frac{\partial h_u}{\partial x_v} \|$ of GraphMinNet does not
decay as $spd(u, v)$ grows (with $n$ tending to $\infty$), where $spd(u, v)$ is the shortest path distance between $u$ and $v$.
\end{proposition}
\begin{proof}
We only consider nonnegative adjacency matrix $A$ and $u \neq v$ since $spd(u,v) = 0$ for any node $u = v$.
Without loss of generality (w.l.o.g.), we consider the case of $l=m=1$ and $d=n$, where $x_v$, $z_v$, and
$\tilde{h}_v$ are scalars with $z_v = \sigma(c_1 x_v) > 0$, $\tilde{h}_v = c_2 x_v$, and $c_2 > 0$.
Let $\tilde{A} = D^{-1/2}AD^{-1/2}$ be the normalized adjacency matrix with $D$ being the diagonal degree matrix of the original adjacency matrix $A$. 
Then 
\begin{align*}
   A_u &= \phi_1(\Lambda) \odot (Wx_u) = {\texttt{diag}}(\phi_1(\Lambda)) (Wx_u)  \\ 
C_u &= \phi_1(\Lambda) \odot p_u.
\end{align*}
By taking $\oplus_1$ as $\odot$ (the case that it is $+$ can be proved similarly), we have 
\begin{align*}
a_u &= B (A_u \odot C_u) \\
&= B ({\texttt{diag}}(\phi_1(\Lambda))(W x_u)) \odot ({\texttt{diag}}(\phi_1(\Lambda))p_u). 
\end{align*}
Thus, we have 
\begin{align*}
    h_u &= \langle a_u, \bar{a} \rangle \odot z_u \odot \tilde{h}_u  \tag{by Eq. (\ref{eq_graph_expanded})} \\
    & = \sum_{v\in V} \langle B \texttt{diag}(\phi_1(\Lambda))(W x_u), B \texttt{diag}(\phi_1(\Lambda)) (W x_v) \rangle \\
    &\quad \quad \langle C_u, C_v \rangle  \odot z_u \odot \tilde{h}_u \\
    &= x_u  z_u \tilde{h}_u (B{\texttt{diag}(\phi_1(\Lambda)})W)^2 \\
    & \quad \sum_{v\in V} (x_v p_u^T {\texttt{diag}(\phi_1(\Lambda)})^2 p_v) {\tag{as $l=1$}}\\
    &= x_u  z_u \tilde{h}_u (B{\texttt{diag}(\phi_1(\Lambda)})W)^2 \\
    & \quad \sum_{v\in V} (x_v e_u^T \tilde{V}{\texttt{diag}(\phi_1(\Lambda)})^2 \tilde{V}^T e_v) \\
    &= x_u  z_u \tilde{h}_u (B{\texttt{diag}(\phi_1(\Lambda)})W)^2 \\
    & \quad \sum_{v\in V} (x_v e_u^T \phi_1^2(\tilde{V} \Lambda \tilde{V}^T) e_v)\tag{by spectral decomposition}\\
    &= x_u  z_u \tilde{h}_u (B{\texttt{diag}(\phi_1(\Lambda)})W)^2 \sum_{v\in V} (x_v e_u^T \phi_1^2(L) e_v)
\end{align*}
In the above, $e_u$ is the unit vector with the $u$-th element being 1; the second equality follows from the Type 4 definition of inner product in Eq. (\ref{eq_reformulation_inner-product}). While the proof is shown using Type 4, it holds for other types by appropriately choosing $B$. The learnable parameters $B$ and $W$ must be chosen to ensure $B{\texttt{diag}(\phi_1(\Lambda)})W \neq 0$.
We define $\phi(\lambda) = \sum_{k=1}^{n} b_k \lambda^k$ with positive constants $b_k > 0$.
Let $\phi_1^2(L) = \phi(\tilde{A})$.
Then, we have
\begin{align*}
\frac{\partial h_u}{\partial x_v} &= x_u  z_u \tilde{h}_u (B{\texttt{diag}(\phi_1(\Lambda)})W)^2 
e_u^T \phi_1^2(L) e_v \\
&= c_2 (B{\texttt{diag}(\phi_1(\Lambda)})W)^2  x_u^2 z_u \sum_{k=1}^n b_k (\tilde{A}^k)_{u,v}. 
\end{align*}
Let $k = spd(u,v) >0$. Since $(\tilde{A}^k)_{u,v}$ represents the degree-weighted sum of all walks of length $k$ from $u$ to $v$, 
and there exists at least one path of length $spd(u,v)$ between $u$ and $v$, it follows that $(\tilde{A}^k)_{u,v} > 0$. We denote this value by $\gamma$. Therefore, we have
\begin{align*}
\|\frac{\partial h_u}{\partial x_v}\| &= c_2 (B{\texttt{diag}(\phi_1(\Lambda)})W)^2 x_u^2 z_u | \sum_{k = 1}^n b_k \tilde{A}_{u,v}^k| \\
&\geq c_2 (B{\texttt{diag}(\phi_1(\Lambda)})W)^2 x_u^2 z_u b_k \gamma > 0.
\end{align*}
The last inequality holds since $b_k > 0$ and $(\tilde{A}^k)_{u,v} \geq 0$ for all $k$, making all terms in the sum non-negative.
This lower bound is independent of the distance $spd(u,v)$, proving that the gradient does not decay with distance and establishing GraphMinNet's long-range dependence property.
\end{proof}

Assuming that the eigenvalue decomposition of the adjacent matrix $A$ is precomputed and thus given, we have the following computational 
complexity for GraphMinNet: 
\begin{proposition}
The hidden states $h_1, \cdots, h_n$ can be computed from $x_1, \cdots, x_n$ with a complexity of $O(nmdl)$, where $n$ is the number of nodes, $m$ is the number of columns in node encoding ($A_u$, $C_u$, or $c_u$),
$d$ is the dimension of rows in node encoding, and $l$ is the feature dimension.

Additionally, the GraphMinNet algorithm in Equations (\ref{eq_graph_expanded}) or (\ref{eq_GraphMinNet_with_self}) achieves linear scalability with respect to the number of nodes.
\end{proposition}
\begin{proof}
{\emph{For linear complexity}}: This can be straightforwardly counted as follows. 
 For each node $u \in V$:
   \begin{enumerate}
       \item Computing $z_u, \tilde{h}_u \in \mathcal{R}^l$:\\
Linear projections and activation $\sigma(\cdot)$: $O(l^2)$
\item Computing node encodings:\\
$Wx_u$ with $W \in \mathcal{R}^{d \times l}$: $O(dl)$\\
$A_u$: $m$ columns × $d$ multiplications = $O(dm)$\\
$C_u$: Similarly $O(dm)$\\
$a_u = B(C_u \odot A_u)$:\\
Element-wise product $C_u \odot A_u$: $O(dm)$\\
Matrix multiplication with $B$: $O(ldm)$
\item Computing final output:\\
Inner product $\langle a_u, \bar{a}\rangle$: $O(ldm)$
\end{enumerate}
Total complexity per node: $O(l^2 + dl + dm + ldm) = O(ldm)$ (assuming $m, d > l$) 

Overall complexity for $n$ nodes: $O(nldm)$.

{\emph{For linear scalability}}: Given eigenvalue decomposition, computations are node-wise independent except for $\bar{a} = \sum_{v \in V} a_v$. The final step requires only one inner product per node with this global sum. Thus, the algorithm scales linearly with $n$ through:

   Independent node-wise computations
   
Single global aggregation

Final node-wise inner products. 
\end{proof}
The above property provides a certificate guaranteeing the linear complexity as well as the linear scalability with respect to the number of nodes of the graph. 

\begin{proposition}
The formulation of GraphMinNet is more powerful than WL test but not more powerful than 3-WL test. 
\end{proposition}
\begin{proof}
First, we show GraphMinNet is more powerful than 1-WL. Let $B = I$, $d=l$, and $\phi(\Lambda_d) = \Lambda_d^{1/2}$. Then:
\begin{align*}
\langle a_u, a_v\rangle &= \langle A_u \odot C_u, A_u \odot C_u\rangle \\
&= \langle C_u, C_v\rangle \langle A_u, A_v\rangle = A_{u,v} \langle A_u, A_v\rangle
\end{align*}

The second equality follows from Type 2 definition of matrix inner product (see Eq. (\ref{eq_reformulation_inner-product})). Since $A_{u,v}$ is an adjacency matrix element and $\langle A_u, A_v\rangle$ is a function of node features $x_u$ and $x_v$, GraphMinNet can implement standard message passing. As message passing is equivalent to the 1-WL test \cite{xu2018powerful}, GraphMinNet is at least as powerful as 1-WL. The additional structural components make it strictly more powerful.

For the upper bound, note that GraphMinNet is an eigenspace projection GNN using a basis invariant function of positional encoding. By \cite{zhang2024expressive}, such architectures cannot exceed the power of 3-WL test.
\end{proof}
\begin{table*}[t]
\caption{\label{tab:dataset}Dataset statistics used in the experiments.}
\resizebox{\linewidth}{!}{
\begin{tabular}{lccccc}
\toprule
{Dataset}         & {\# Graphs} & {Avg. \# nodes} & {Avg. \# edges} & {Prediction level} & {Prediction task}         \\ 
\midrule
ZINC     & 12,000   & 23.2         & 24.9         & graph            & regression              \\
ogbg-molhiv     & 41,127   & 25.5         & 27.5         & graph            & binary classification   \\
\midrule
MNIST           & 70,000   & 70.6         & 564.5        & graph            & 10-class classification \\
CIFAR10         & 60,000   & 117.6        & 941.1        & graph            & 10-class classification \\
PATTERN         & 14,000   & 118.9        & 3,039.3      & node             & binary classification   \\
CLUSTER         & 12,000   & 117.2        & 2,150.9      & node             & 6-class classification  \\
\midrule
MalNet-Tiny    & 5,000   &  1,410.3        & 2,859.9      & graph             & 5-class classification \\ 
\midrule
Peptides-func   & 15,535   & 150.9        & 307.3        & graph            & 10-class classification \\
Peptides-struct & 15,535   & 150.9        & 307.3        & graph            & regression              \\
PascalVOC-SP    & 11,355   & 479.4        & 2,710.5      & node             & 21-class classification \\ 
\bottomrule
\end{tabular}
}
\end{table*}
\section{Datasets}

\label{appendix:dataset}
We utilized 10 datasets, which are widely adopted in the graph machine learning community. These datasets cover a range of tasks, including graph-level regression, binary classification, and node-level classification. All datasets utilized in our study are equipped with predefined training, validation, and test splits, ensuring consistency across experiments. In line with established practices in the field, we report the test results based on the best-performing models on the validation set. To ensure the robustness of our findings, we conduct evaluations over five distinct random seeds for each dataset. This approach aligns with methodologies outlined in prior studies \citep{rampavsek2022recipe, ma2023graph, shirzad2023exphormer,huang2024can}.
\paragraph{ZINC:}The ZINC-12k dataset \citep{dwivedi2023benchmarking} is a subset of the ZINC database, which contains commercially available chemical compounds. This dataset comprises 12,000 molecular graphs, where each graph represents a small molecule with the number of atoms ranging from 9 to 37. In this representation, nodes correspond to heavy atoms (with 28 distinct atom types), and edges symbolize chemical bonds (of 3 different types). The primary task associated with ZINC-12k is graph-level regression.
\paragraph{ogbg-molhiv:} The ogbg-molhiv dataset \citep{hu2020open} is adopted from the MoleculeNet collection \citep{wu2018moleculenet} by the Open Graph Benchmark (OGB) project. It consists of molecular graphs where nodes and edges are featurized to represent various chemophysical properties. The task for this dataset is a binary graph-level classification, aiming to predict whether a molecule can inhibit HIV replication.
\paragraph{MNIST \& CIFAR10:} The MNIST and CIFAR10 datasets \citep{dwivedi2023benchmarking} are derived from well-known image classification datasets. In these graph versions, each image is converted into a graph by constructing an 8-nearest-neighbor graph of SLIC superpixels \citep{achanta2012slic} for each image. The task for these datasets is a 10-class graph-level classification, mirroring the original image classification tasks. 
\paragraph{PATTERN and CLUSTER:} The PATTERN and CLUSTER datasets \citep{dwivedi2023benchmarking} are synthetic datasets that model community structures using the Stochastic Block Model (SBM). Both tasks involve inductive node-level classification. In the PATTERN dataset, the goal is to identify nodes that belong to one of 100 randomly generated subgraph patterns, which are distinct from the rest of the graph in terms of SBM parameters. For the CLUSTER dataset, each graph is composed of 6 clusters generated by the SBM, and the task is to predict the cluster ID of 6 test nodes, each representing a unique cluster within the graph.
\paragraph{MalNet-Tiny:} The MalNet-Tiny dataset \citep{freitas2020large} is a subset of the larger MalNet repository, which contains function call graphs extracted from Android APKs. This subset includes 5,000 graphs, each with up to 5,000 nodes, representing either benign software or four categories of malware. In the MalNet-Tiny subset, all original node and edge attributes have been removed, and the classification task is based solely on the graph structure.
\paragraph{Peptides-func and Peptides-struct:} The Peptides-func and Peptides-struct datasets \citep{dwivedi2022long} are derived from 15,000 peptides retrieved from the SATPdb database \citep{singh2016satpdb}. Both datasets use the same set of graphs but focus on different prediction tasks. Peptides-func is a graph-level classification task with 10 functional labels associated with peptide functions. In contrast, Peptides-struct is a graph-level regression task aimed at predicting 11 structural properties of the molecules. 
\paragraph{PascalVOC-SP:} The PascalVOC-SP dataset \citep{dwivedi2022long} is a node classification dataset based on the Pascal VOC 2011 image dataset \citep{everingham2010pascal}. Superpixel nodes are extracted using the SLIC algorithm \citep{achanta2012slic}, and a rag-boundary graph that interconnects these nodes is constructed. The task is to classify each node into corresponding object classes, akin to semantic segmentation. 
\begin{table*}[t]
\caption{\label{tab:hyp}Hyperparameters used to achieve our result.}
\resizebox{\linewidth}{!}{
\begin{tabular}{lcccccccc}
\toprule
{Dataset}         & Hidden dimension $(l)$ & {Local Method} & {\# Layers} & {Lap. dim} & {Batch} & {LR}   & {Weight Decay} & Dropouts\\ 
\midrule
ZINC&$64$&GINE&$10$&$16$&$32$&$0.001$&$1e^{-5}$&$\{0.1,0.0,0.6,0.0\}$ \\
ogbg-molhiv &$64$&GatedGCN&$6$&$16$&$128$&$0.002$&$0.001$&$\{0.1,0.3,0.1,0.1\}$\\
MNIST&$52$&GatedGCN&$3$&$32$&$16$&$0.005$&$0.01$&$\{0.1,0.1,0.1,0.4\}$\\
CIFAR10 &$52$&GatedGCN&$3$&$32$&$16$&$0.005$&$0.01$&$\{0.1,0.1,0.1,0.0\}$\\
PATTERN &$36$&GatedGCN&$24$&$32$&$32$&$0.001$&$0.1$&$\{0.1,0.1,0.5,0.0\}$\\
CLUSTER&$36$&GatedGCN&$24$&$32$&$16$&$0.001$&$0.1$&$\{0.1,0.3,0.3,0.0\}$ \\
MalNet-Tiny&$64$&GatedGCN&$5$&$0$&16&0.0015&$1e^{-5}$&$\{0.1,0.1,0.35,0.05\}$ \\ 
Peptides-func&$100$&GatedGCN&$3$&$31$&$256$&$0.003$&$0.1$&$\{0.0,0.1,0.1,0.3\}$ \\
Peptides-struct&$100$&GatedGCN&$3$&$31$&$128$&$0.003$&$0.1$&$\{0.0,0.1,0.1,0.5\}$ \\
PascalVOC-SP&$96$&GatedGCN&$4$&$63$&$16$&$0.002$&$0.1$&$\{0.0,0.0,0.5,0.0\}$ \\ 
\bottomrule
\end{tabular}
}
\end{table*}
\section{Hyperparameters}
\label{sec:hyerparameters}
In this section, we summarize the key hyperparameters used to achieve the results presented in Table~\ref{tab:main_result}. These parameters, detailed in Table~\ref{tab:hyp}, include the hidden dimension ($l$), the type of local method (such as GINE or GatedGCN), the number of layers, the dimensionality of the Laplacian features (eigenvalues and eigenvectors), batch size, learning rate, weight decay, and dropout rates. The dropout rates are specified for different components of the model: the feedforward network, the local method, the residual, and our proposed GraphMinNet. For MalNet-Tiny and PascalVOC-SP datasets, we use the 0-th order power function $\phi_i$ in $A_u$. Additionally, for MalNet-Tiny, we set $\phi_i$ in $C_u$ to be the zero function and define the $\oplus_1$ operation as elementwise addition. For all the other datasets, we used a generalized permutation equivariant set aggregation function~\cite{wang2022equivariant}.


\end{document}